\documentclass[twoside,11pt]{article}
\usepackage{jair, theapa, rawfonts}

\usepackage{graphicx}
\usepackage{subfigure}
\usepackage{amsmath}
\usepackage{amssymb}
\usepackage{amsthm}
\usepackage{algorithm}
\usepackage{algorithmicx}
\usepackage{algpseudocode}
\usepackage{wrapfig}

\usepackage{caption2}

\usepackage{color}

\ShortHeadings{Planning from Scratch}
{Yunlong Liu, Jianyang Zheng}
\firstpageno{1}

\begin{document}
\newtheorem{defn}{Definition}
\newtheorem{thm}{Theorem}
\newtheorem{lem}{Lemma}
\newtheorem{pro}{Proposition}

\title{Combining Offline Models and Online Monte-Carlo Tree Search for Planning from Scratch}

\author{\name Yunlong~Liu* \email ylliu@xmu.edu.cn \\
       \name Jianyang~Zheng \\
       \addr Department of Automation, Xiamen University,\\
       Xiamen 361005, China
       }


\maketitle

\begin{abstract}
Planning in stochastic and partially observable environments is a central issue in artificial intelligence. One commonly used technique for solving such a problem is by constructing an accurate model firstly. Although some recent approaches have been proposed for learning optimal behaviour under model uncertainty, prior knowledge about the environment is still needed to guarantee the performance of the proposed algorithms. With the benefits of the Predictive State Representations~(PSRs) approach for state representation and model prediction, in this paper, we introduce an approach for planning from scratch, where an offline PSR model is firstly learned and then combined with online Monte-Carlo tree search for planning with model uncertainty. By comparing with the state-of-the-art approach of planning with model uncertainty, we demonstrated the effectiveness of the proposed approaches along with the proof of their convergence. The effectiveness and scalability of our proposed approach are also tested on the RockSample problem, which are infeasible for the state-of-the-art BA-POMDP based approaches.
\end{abstract}

\section{Introduction}
\label{Introduction}

A central problem in artificial intelligence is for agents to find optimal policies in stochastic, partially observable environments, which is an ubiquitous and challenging problem in science and engineering. One commonly used technique for solving such partially observable problems is to model the dynamics of the environments firstly, for example, the Partially Observable Markov Decision Processes~(POMDP)~\cite{Kaelbling98planningand,Ross2008OnlinePOMDP} and Predictive State Representations~(PSRs)~\cite{Littman01predictiverepresentations,LiuAAMAS15,Liu16IJCAI,TalvitieS11} approach, and then the problem can be solved using the obtained model. Although POMDPs and PSRs provide general frameworks to solve partially observable problems, they rely heavily on a known and accurate model of the environment~\cite{Liu14inaccuratePSR,Spaan2005,Pineau2006,Ye2017}. However, in real-world applications it is extremely difficult to build an accurate model.

Some Bayesian approaches have been proposed to address the problem of planning with model uncertainty by incorporating prior knowledge of the environments into a prior distribution over the unknown model parameters, a posterior distribution on these parameters is updated as the agent performs actions and gets observations from the environment, then the agent can compute an optimal policy under the current posterior distribution~\cite{Duff2002phdthesis,Guez2013,Ghavamzadeh2015BRLSurvey}. For the partially observable environments, the Bayesian approach is casted into a POMDP problem, namely Bayes-Adaptive Partially Observable Markov Decision Processes~(BA-POMDPs), by treating the distribution on model information as the part of the hidden states~\cite{Ross2008BAPOMDP,Ross2011BAPOMDP}. Unfortunately, with the increase of the number of hidden states, the problem becomes more complex than the original one. As is well known, finding an (approximate) optimal POMDP solution is difficult, such a casting usually leads to intractable large number of possible model states and model parameters, and the related approaches can only be applied to some trivial problems.

Developments in online and sample-based planning have achieved high performance in larger scale systems under the assumption that the accurate models of the underlying systems are known a prior~\cite{Ross2008OnlinePOMDP,Silver10monte-carloplanning}. In the work of~\cite{Katt2017MCTS,Katt2018}, Monte-Carlo tree search~(MCTS) based Partial Observable Monte-Carlo Planning~(POMCP)~\cite{Silver10monte-carloplanning} was extended to the Bayes-Adaptive case, leading to an approach: BA-POMCP, for the case of planning under model uncertainty and resulting applications in larger problems. However, strong prior knowledge about the environment is still needed to guarantee the performance of the proposed approach.

Predictive State Representations~(PSRs) offer a powerful framework for modelling partially observable dynamical systems~\cite{Littman01predictiverepresentations,Boots-2011b,Huang2018}. Unlike latent-state based approaches, such as hidden Markov models~(HMM) and POMDPs, PSRs represent state as predictions about future observable events, which leads to easier learning of the corresponding model, the avoidance of using local-minima prone expectation maximization, more expressive power, etc.~\cite{Singh04predictivestate}. Moreover, rather than usually requiring a predetermined latent state structure as an input for latent-state based approaches, when learning the PSR model, no such prior knowledge about the environment is needed~\cite{JMLR:hamilton14}.

In this paper, with the benefits of PSRs for model learning and updating, and by combining the online Monte-Carlo tree search, we introduce an approach for planning with model uncertainty, where the planning process starts from scratch and no prior knowledge of the underlying system is required. We divide the planning process into two stages: 1) By directly treating some reward signals of the underlying system as the observations of the environment, we firstly learn a PSR model using the training data; 2) The learned PSR model is combined with the online Monte-Carlo tree search, where the learned PSR model is used as the simulator for computing good local policies at each decision step during the execution of online Monte-Carlo tree search. The effectiveness of the proposed approach is demonstrated by the comparison with the-state-of-the-art approach: BA-POMCP~\cite{Katt2017MCTS,Katt2018}. Moreover, we prove the correctness and convergence of the proposed approach along with the analysis of the advantage of our algorithm.

\section{Background}
\subsection{Predictive State Representations}

Predictive State Representations~(PSRs) represent state by using a vector of predictions of fully observable quantities~(tests) conditioned on past events~(histories), denoted $b(\cdot)$. For discrete systems with finite set of observations $O=\{o^1,o^2,\cdots,o^{|O|}\}$ and actions $A=\{a^1,a^2,\cdots,a^{|A|}\}$, at time $\tau$, a \emph{test} is a sequence of action-observation pairs that starts from time $\tau + 1$. Similarly, a \emph{history} at $\tau$ is a sequence of action-observation pairs that starts from the beginning of time and ends at time $\tau$, which is used to describe the full sequence of past events. The prediction of a length-$m$ test $t$ at history $h$ is defined as $p(t|h)=p(ht)/p(h)=\prod^m_{i=1}Pr(o_i|ha_1o_1\cdots a_i)$~\cite{Singh04predictivestate,Huang2018}.

The underlying dynamical system can be described by a special bi-infinite matrix, called the Hankel matrix~\cite{Balle:2014ML}, where the rows and columns correspond to all the possible tests $\mathcal{T}$ and histories $\mathcal{H}$ respectively, the entries of the matrix are defined as $P_{t,h}=p(ht)$ for any $t \in \mathcal{T}$ and $h \in \mathcal{H}$, where $ht$ is the concatenation of $h$ and $t$~\cite{Boots-2011b}. The rank of the Hankel matrix is called the linear dimension of the system. When the rank is finite, we assume it is $k$, in PSRs, the state of the system at history $h$ can be represented as a prediction vector of $k$ tests conditioned at $h$. The $k$ tests used as the state representation is called the minimal {\em core tests} that the predictions of these tests contain sufficient information to calculate the predictions for all tests, and is a \emph{sufficient statistic}. For linear dynamical systems, the minimal core tests can be the set of tests that corresponds to the $k$ linearly independent columns of the Hankel matrix~\cite{Singh04predictivestate}.

Directly discovering the set of core tests is usually difficult and time-consuming, spectral approaches have been proposed to alleviate the discovery problem by specifying a large enough set of tests so that it almost certainly contains a set of core tests~\cite{Boots-2011b,Huang2018}. For the spectral approach, a PSR of rank $k$ can be parameterized by a reference condition state vector $b_* = b(\epsilon)\in \mathbb{R}^k$, an update matrix $B_{ao} \in \mathbb{R}^{k \times k}$ for each $a \in A$ and $o \in O$, and a normalization vector $b_\infty \in \mathbb{R}^k$, where $\epsilon$ is the empty history and $b_\infty^T B_{ao} = 1^T$~\cite{Hsu_aspectral,Boots-2011b}. In the spectral approach, these parameters can be defined in terms of the matrices $P_\mathcal{H}$, $P_{\mathcal{T,H}}$, $P_{\mathcal{T},ao,\mathcal{H}}$ and an additional matrix $U \in \mathbb{R}^{|\mathcal{T}| \times |k|}$ as shown in Eq.~\ref{equ:spectral}, where $\mathcal{T}$ and $\mathcal{H}$ are the set of all possible tests and histories respectively, $P_\mathcal{H}$ contains the probabilities of every $h \in \mathcal{H}$, entries of $P_{\mathcal{T,H}}$ are joint probabilities of tests $t \in \mathcal{T}$ and $h \in \mathcal{H}$, $U$ is the left singular vectors of the matrix $P_{\mathcal{T,H}}$, $^T$ is the transpose and $\dag$ is the pseudo-inverse of the matrix~\cite{Boots-2011b}.
\begin{equation}
\begin{aligned}
 & b_* = U^TP_{\mathcal{T,H}}1_k,\\
 &b_{\infty} = (P^T_{\mathcal{T,H}}U)^{\dag}P_\mathcal{H},\\
 &B_{ao} = U^TP_{\mathcal{T},ao,\mathcal{H}}(U^TP_{\mathcal{T,H}})^\dag.
\label{equ:spectral}
\end{aligned}
\end{equation}

Using these parameters, after taking action $a$ and receiving observation $o$ at history $h$, the PSR state at next time step $b(hao)$ is updated from $b(h)$ as follows~\cite{Boots-2011b}:
\begin{equation}
b(hao) = \frac{B_{ao}b(h)}{b_\infty^TB_{ao}b(h)}.
\label{equ:nextB}
\end{equation}

Also, the probability of observing the sequence $a_1o_1a_2o_2 \cdots a_no_n$ in the next $n$ time steps can be predicted by~\cite{Boots-2011b}:
\begin{equation}
Pr[o_{1:t}||a_{1:t}] = b_{\infty}^TB_{a_no_n} \cdots B_{a_2o_2}B_{a_1o_1}b_*.
\end{equation}

\subsection{Monte-Carlo Tree Search}

Monte-Carlo tree search method finds optimal decisions in a decision space by combining Monte-Carlo simulation with game tree search~\cite{GellyACM}. It iteratively builds a search tree by adding new nodes to the existing search tree until some predefined condition is reached~\cite{Liu16IJCAI}. Each node $T$ in the tree corresponds to a state $s$, and contains an action value $Q(s,a)$, a visitation count $N(s,a)$ for each action $a \in A$~\cite{Gelly2011AI}. 

Monte-Carlo simulation is used to compute state-action values $Q(s,a)$, where each simulation contains two stages: a tree policy and a rollout policy. When state $s$ is represented in the existing search tree, the tree policy is used to select actions. Once a simulation leaves the scope of the existing search tree, the rollout policy is used until the termination of the simulation. After each simulation, one new node that is first visited in the second stage is added to the search tree.
Then $Q(s,a)$ in the search tree is the mean outcome of all simulations starting from $s$ in which action $a$ was selected in state $s$~\cite{BrowneSurveyMCTS2012}:
\begin{equation}
Q(s,a)=\frac{1}{N(s,a)}\sum_{i=1}^{N(s)}\mathbb{I}_i{(s,a)}z_i,
\end{equation}
where $\mathbb{I}_i{(s,a)}$ is an indicator function returning 1 if action $a$ was selected in state $s$ during the $i^{th}$ simulation, and 0 otherwise; $z_i$ is the outcome of the $i^{th}$ simulation. 

The basic form of MCTS just selects the greedy action with the highest value during the first stage and selects actions uniformly at random during the second stage. Such a strategy can often be inefficient in constructing a search tree. By treating the choice of actions as a multi-armed bandit problem, Kocsis et al.~(2006) ~\cite{Kocsis:2006} proposed the use of the UCB1 algorithm for action selection in the search tree of MCTS, namely, the UCT algorithm. The tree policy selects the action $a^*$ maximizing the augmented value, which allows for an optimal trade-off between exploitation and exploration:
\begin{equation}
\begin{aligned}
 &Q^\oplus(s,a)=Q(s,a)+c\sqrt{\frac{\log{N(s)}}{N(s,a)}},\\
 &a^*=\arg\underset{a}\max{\ Q^\oplus(s,a)},
\end{aligned}
\end{equation}
where $c>0$ is the exploration constant and $N(s) =\sum_a N(s,a)$. As can be seen, the action value is augmented by an exploration bonus that is the largest for the actions that have been tried the least number of times and therefore the most uncertain, which allows for an optimal trade-off between exploitation and exploration.

In partially observable environments, where state $s$ cannot be directly observed, history $h$ is used as state representation and at each time step, online planning is performed by incrementally building a lookahead tree with node $T(h)$ that contains $N(h)$,$N(h,a)$, and $V(h,a)$~\cite{Silver10monte-carloplanning}.

\section{Planning via Offline Models and Online Search}

In this section, we first show how the learned PSR model is combined with Monte-Carlo tree search to realize the planning from scratch, and then we prove the convergence of the proposed approach.

\subsection{Plan from Scratch}

The most practical solution for solving the problem of online planning in partially observable environments is to extend Monte-Carlo tree search to the model of the environment~\cite{Silver10monte-carloplanning,Katt2017MCTS}. To realize online planning, at each decision step, a lookahead tree through simulated experiments is constructed to form a local approximation to the optimal value function. However, the model used for generating the simulated experiments is usually assumed to be accurate, which may be impossible in real-world applications~\cite{Silver10monte-carloplanning}. For the BA-POMDP and BA-POMCP approaches~\cite{Ross2011BAPOMDP,Katt2017MCTS}, although the model is learned during execution, to guarantee the performance, strong prior knowledge that the nearly correct initial models of the environments is still required. At the same time, to find the local optimal action at each decision step, knowledge of reward at each state after taking some action is required for estimating the value of each node~\cite{Silver10monte-carloplanning,Katt2017MCTS}. For the partially observable environments, in many cases, as we may not know the states of the underlying system, the reward of the state after taking some action is also hard to know.

\begin{algorithm}[!htb]
\caption{PSR-MCTS}
\label{Algo:PSR-MCTS}
\begin{algorithmic}\small
\State $h \leftarrow ()$
\State $b(h)\leftarrow \hat{b}_*$
\Repeat
\State $a \leftarrow$ Act-Search($b(h),n\_sims,h$)
\State{\bfseries EXECUTE} $a$
\State $o \leftarrow$ observation received from the world
\State $b(hao) = \frac{\hat{B}_{ao}b(h)}{\hat{b}_\infty^T\hat{B}_{ao}b(h)}$
\State $h \leftarrow hao$
\Until{the end of a plan}

\end{algorithmic}
\end{algorithm}

\begin{algorithm}[!htb]
\caption{Act-Search($b(h)$,$n\_sims$,$h$)}
\label{Algo:Action-Search}
\begin{algorithmic}\small
\State $h_0 \leftarrow h$
\State $\bar{b}(h_0) \leftarrow Copy(b(h))$
\For{$i \leftarrow 1$ {\bfseries to} $n\_sims$}
\State $Simulate(\bar{b}(h_0),0,h_0)$
\EndFor
\State $a \leftarrow GreedyActionSelection(h_0)$
\State{\bfseries return} $a$

\end{algorithmic}
\end{algorithm}

\begin{algorithm}[tb]
\caption{Simulate($b(h),depth,h$)}
\label{Algo:Simulate}
\begin{algorithmic}\small
\If {$depth == max\_dep \| IsTerminal(h)$}
\State {\bfseries return} $0$
\EndIf
\State $//$Select action according to the UCT algorithm\cite{Kocsis:2006}
\State $a \leftarrow$ UCBACTIONSELECTION($h$)
\State $o \leftarrow$ sampled according to Equ.~\ref{equ:obserDistri}
\If{$o$ corresponds to some reward}
\State $R \leftarrow reward(o)$
\Else
\State $R \leftarrow reward(ao)$
\EndIf
\State $h' \leftarrow hao$
\State $b(h') = \frac{\hat{B}_{ao}b(h)}{\hat{b}_\infty^T\hat{B}_{ao}b(h)}$
\If{$h' \in Tree$}
\State $r \leftarrow R+\gamma\cdot$Simulate($b(h'),depth+1,h'$)
\Else
\State ConstructNode($h'$)
\State $r \leftarrow R+\gamma\cdot$RollOut($b(h'),depth+1,h'$)
\EndIf

\State $//$Update statistics
\State $N(h) \leftarrow N(h)+1$
\State $N(h,a) \leftarrow N(h,a)+1$
\State $V(h,a) \leftarrow V(h,a)+\frac{r-V(ha)}{N(ha)}$
\State{\bfseries return $r$}

\end{algorithmic}
\end{algorithm}

\begin{algorithm}[tb]
\caption{RollOut($b(h),depth,h$)}
\label{Algo:RollOut}
\begin{algorithmic}\small
\If {$depth == max\_dep \| IsTerminal(h)$}
\State {\bfseries return} $0$
\EndIf
\State $a \leftarrow \pi_{rollout}(h)$
\State $o \leftarrow$ sampled according to Equ.~\ref{equ:obserDistri}
\If{$o$ corresponds to some reward}
\State $R \leftarrow reward(o)$
\Else
\State $R \leftarrow reward(ao)$
\EndIf
\State $h' \leftarrow hao$
\State $b(h') = \frac{\hat{B}_{ao}b(h)}{\hat{b}_\infty^T\hat{B}_{ao}b(h)}$
\State $r \leftarrow R+\gamma\cdot$RollOut($b(h'),depth+1,h'$)
\State{\bfseries return} $r$

\end{algorithmic}
\end{algorithm}
As mentioned previously, PSRs are powerful methods for modelling dynamical systems, which represent state using predictions of actually happened actions and observations~\cite{Boots-2011b}. Compared to the POMDP approach, PSRs are easier to learn and require no prior knowledge. Moreover, given an action executed, using a PSR model to compute the possibility of next observation and next state representation is more computation efficient than using a POMDP model~(the detail is given in the next subsection), which is crucial for state updating and generating simulated experiments at each decision step.

With the benefits of PSRs for model learning and updating, and by treating some rewards encountered in the interaction with the environment directly as observations, we introduce an approach, namely PSR-MCTS, for planning from scratch by combining the offline learned PSR model and online Monte-Carlo tree search, where only training data is used and no prior knowledge about the underlying system is required.

The approach is divided into two stages. In the first stage, a PSR model is learned using the training data by building empirical estimates $\hat{P}_\mathcal{H}$,$\hat{P}_{\mathcal{T,H}}$, and $\hat{P}_{\mathcal{T},ao,\mathcal{H}}$ of the matrices $P_\mathcal{H}$, $P_{\mathcal{T,H}}$, and $P_{\mathcal{T},ao,\mathcal{H}}$ defined above. Then, $\hat{U}$ can be computed by singular value decomposition of $\hat{P}_{\mathcal{T,H}}$, and the parameters can be computed as follows~\cite{Boots-2011b}:
\begin{equation}
\begin{aligned}
 \label{equ:spectlearn}
 & \hat{b}_* = \hat{U}^T\hat{P_{\mathcal{T,H}}}1_k,\\
 & \hat{b}_{\infty} = (\hat{P}^T_{\mathcal{T,H}}\hat{U})^{\dag}\hat{P}_\mathcal{H},\\
 & \hat{B}_{ao} = \hat{U}^T\hat{P}_{\mathcal{T},ao,\mathcal{H}}(\hat{U}^T\hat{P}_{\mathcal{T,H}})^\dag.
\end{aligned}
\end{equation}
With the increase of the training data, the estimate of $\hat{P}_\mathcal{H}$,$\hat{P}_{\mathcal{T,H}}$, and $\hat{P}_{\mathcal{T},ao,\mathcal{H}}$  can be guaranteed to be converged to the true matrices $P_\mathcal{H}$, $P_{\mathcal{T,H}}$, and $P_{\mathcal{T},ao,\mathcal{H}}$ by the law of large numbers. Then for a PSR of finite rank, the parameters $\hat{b}_*$, $\hat{b}_{\infty}$, and $\hat{B}_{ao}$ can converge to the true parameters~\cite{Boots-2011b}.
The second stage extends Monte-Carlo tree search to the obtained PSR model for online planning. At each decision step $h$ during the execution, firstly, the current state representation $b(h)$ is computed according to Equ.~\ref{equ:nextB}, then simulated experiments starting from a copy of $b(h)$ are generated to construct a lookahead search tree for computing good local polices. In the first stage of simulation, if all possible child action nodes exist, then action is selected to maximise $\small{V^\oplus(h,a)=V(h,a)+c\sqrt{\frac{\log{N(h)}}{N(h,a)}}}$, i.e., $\small{a^*=\arg\underset{a}\max{\ V^\oplus(h,a)}}$. In the second stage of simulation, actions are selected by an uniform randomly history based rollout policy $\pi_{rollout}(h)$~\cite{Silver10monte-carloplanning}. For both these two stages, after action $a$ is selected, the next observation $o$ is sampled according to the following distribution:
\begin{equation}
Pr[o||ha] = \hat{b}_{\infty}^T\hat{B}_{ao}b(h), \forall o \in O.
\label{equ:obserDistri}
\end{equation}
Then the next state representation $b(hao)$ is computed using the learned PSR model. This process continues to execute until the termination of the simulation, and the related statistics contained in each visited node, e.g., $N(h,a)$, $V(h,a)$, are updated accordingly. As we treat the rewards directly as observations, in the simulation, when the sampled observation $o$ is the reward that indicates the end of a process, the simulation ends. Otherwise, the simulation ends with some pre-defined conditions. Note that some state-independent rewards, such as the rewards received at every time step or action-only-dependent rewards in some domains, are not treated as observations and not used for the model learning. When the search is complete, action $a$ with the greatest value is executed, and a real observation $o$ from the world is received, then $h \leftarrow hao$, $b(h)$ is updated according to Equ.~\ref{equ:nextB}, and the node $T(h)$ becomes the root of the new search tree. The complete approach is described in Algorithm~\ref{Algo:PSR-MCTS}$\sim$\ref{Algo:RollOut}, where $reward(ao)$ is the reward that is not treated as observation, $\gamma$ is a discounted factor specified by the environment, $n\_sims$ is the number of simulations used for finding the executed action at each step, $max\_dep$ and $IsTerminal(h)$ are some predefined conditions for the termination.

\subsection{Theoretical Analysis}

As can be seen from Algorithm~\ref{Algo:PSR-MCTS}$\sim$~\ref{Algo:RollOut}, the PSR-MCTS approach involves two main computations. The first is the computation of $b(hao)$ and the second is the generation of next observation $o$ at each decision/simulation step. Here we first show besides the advantage that the plan can be realized from scratch, such two computations of the PSR-MCTS approach are affordable. Moreover, compared to the POMDP approaches, the proposed approach is usually more computation efficient, then we prove the proposed approach converges to the optimal value function under some conditions.
\begin{thm}
The size of the state representation of the PSR model is much smaller than the size of the state representation, i.e., belief state, of the BA-POMDP model for the same system.
\label{Thm:size}
\end{thm}
\begin{proof}
As the size of the state representation of the PSR model, $k$, is no larger than the number of states in the \textbf{minimal} POMDP model of the same system~\cite{Littman01predictiverepresentations}, and the number of parameters of one augmented state $\bar{s}=<s,\chi>$ of the BA-POMDP is up to $|S|^2 \times |A| + |S|\times|A|\times|O|$, moreover, the number of possible augmented states in BA-POMDP grows exponentially with time~\cite{Ross2011BAPOMDP}. Thus, $k$ is much smaller than the size of the belief state in BA-POMDP.
\end{proof}

\begin{lem}
The computation of the probability of an observation $o$ and the next state representation is more efficient by using a PSR model than using the POMDP approach.
\end{lem}
\begin{proof}
With a PSR of rank $k$, $\small{Pr[o||ha] = b_{\infty}^TB_{ao}b(h)}$, where $b$ is a $1 \times k$ vector and $B_{ao}$ is a $k \times k$ matrix, while for the POMDP approach with $n$ states, $\small{Pr[o||ha] = \mathrm{b}^T(h)T^aZ^{ao}1^n}$, where $T$ and $Z$ are $n \times n$ matrix and $n \ge k$. Thus, the computation is more efficient by using the PSR model, so as the computation of next state representation.
\end{proof}
\begin{lem}
Compared to the original PSR model, the rank of a PSR model by adding rewards as observations, so as the dimension of $b$ and $B_{ao}$, is still upper bounded by the number of states in the \textbf{minimal} POMDP model of the system.
\end{lem}
\begin{proof}
Construct a ($\mathcal{H} \times |S|$) matrix $B$ and a ($|S| \times \mathcal{T}$) matrix $D$, where $|S|$ is the number of states in the \textbf{minimal} POMDP model, row $i$ of $B$ is the belief-state corresponding to history $h_i$, column $j$ of $D$ is a column vector that contains the prediction of test $t_j$ at each nominal-state so $D_{ij}=p(t_j|s_i)$. Then, $P_{\mathcal{T,H}}$ can be calculated as:$P_{\mathcal{T,H}} = P(\mathcal{H})P(\mathcal{T|H})=P(\mathcal{H})BD$ and the rank of $P_{\mathcal{T,H}}$ is upper bounded by ranks of $B$ and $D$, which is no more than $|S|$.
\end{proof}

In practice, for larger scale systems, according to the computation power and requirement of time limitation on online planning, we can select an appropriate size of $k$ to learn an approximate PSR model to make the computations of $b(hao)$ and the generation of next observation $o$ at each step affordable, as we can also compute $b_{\infty}^TB_{ao}$~($\forall a \in A$ and $o \in O$) offline to reduce the online computation time, this enables the possible application of the proposed approach into larger scale systems.

As for systems that can be represented by a finite POMDP $M = (S,A,T,R,O,Z)$, a PSR $\tilde{M} = (A,\tilde{R},O,B)$ also exists~\cite{Littman01predictiverepresentations}, where $\tilde{R}_h^a=\sum_{s \in S} \mathrm{B}(s,h)R_s^a$ and $\mathrm{B}(s,h)$ is the belief state. Also, $\forall$ $h$, $a \in A$ and $o \in O$, $Pr[o||ha] = b_{\infty}^TB_{ao}b(h) = \mathrm{b}^T(h)T^aZ^{ao}1^n$~\cite{Littman01predictiverepresentations,Boots-2011b}.  Following we will prove the convergence of our proposed approach to the optimal value function. The main steps of these proofs are similar to those in~\cite{Silver10monte-carloplanning}.
\begin{lem}
\label{Lem:val}
Given a POMDP $M$, consider the PSR model $\tilde{M}$ of the same system, the value function $\tilde{V}^\pi(h)$ of the PSR is equal to the value function $V^\pi(h)$ of the POMDP.
\end{lem}
\begin{proof}
$V^\pi(h)=\sum\limits_{s \in S}\sum\limits_{a \in A}\sum\limits_{s' \in S}\sum\limits_{o \in O}\mathrm{B}(s,h)\pi(h,a)(R_s^a+\gamma T^a_{ss'}Z_{s'o}^aV^\pi(hao))=\sum\limits_{a \in A}\sum\limits_{o \in O}\pi(h,a)(\tilde{R}_h^a +\gamma Pr[o||ha]V^\pi(hao))=\sum\limits_{a \in A}\sum\limits_{o \in O}\pi(h,a)(\tilde{R}_h^a+\gamma b_{\infty}^TB_{ao}b(h)\tilde{V}^\pi(hao))=\tilde{V}^\pi(h).$
\end{proof}
Assume two distributions, $D_\pi(h_T)$ and $\tilde{D}_\pi(h_T)$,  where $D_\pi(h_T)$ is the POMDP rollout distribution and $\tilde{D}_\pi(h_T)$ is the PSR rollout distribution. For $D_\pi(h_T)$, it is the distribution of histories generated by sampling an initial state $s_t \sim \mathrm{B}(s,h_t)$, and then repeatedly sampling actions from policy $\pi(h,a)$ and sampling states, observations and rewards from $M$, until termination at time $T$~\cite{Silver10monte-carloplanning}. $\tilde{D}_\pi(h_T)$ is the distribution of histories generated by starting $h_t$, and then repeatedly sampling actions from policy $\pi(h,a)$ and sampling observations and rewards from $\tilde{M}$, until termination at time $T$.
\begin{lem}
\label{Lem:dis}
For any rollout distribution, the PSR rollout distribution is equal to the POMDP rollout distribution, i.e., $\forall \pi, D_\pi(h_T)=\tilde{D}_\pi(h_T)$.
\end{lem}
\begin{proof}
$D^\pi(hao)=\\D^\pi(h)\pi(h,a)\sum_{s\in S}\sum_{s' \in S}\mathrm{B}(s,h)T_{ss'}^aZ_{s'o}^a=D^\pi(h)\pi(h,a)Pr[o||ha]=\tilde{D}^\pi(h)\pi(h,a)b_{\infty}^TB_{ao}b(h)=\tilde{D}^\pi(hao).$
\end{proof}

By Lemma~\ref{Lem:val} and~\ref{Lem:dis}, and according to Lemma 1 and 2 of~\cite{Silver10monte-carloplanning}, we can conclude that the value function $\tilde{V}^\pi(h)$ of the PSR is equal to the value function of the derived MDP with histories as states and the PSR rollout distribution is equal to the derived MDP rollout distribution. As the UCT algorithm converges to the optimal value function in fully observable MDPs~\cite{Kocsis:2006} and with infinite training data, the law of large numbers guarantees the learned PSR model converges to the true PSR model, following Theorem 1 in~\cite{Silver10monte-carloplanning}, Lemma~\ref{lem:converge} holds.
\begin{lem}
\label{lem:converge}
With infinite training data and for suitable choice of $c$, the value function constructed by our approach~(PSR-MCTS) converges in probability to the optimal value function, $V(h)\xrightarrow{p} V^*(h)$, for all histories $h$ that are prefixed by $h_t$.
\end{lem}

\section{Experiments}
\noindent{\bf Experimental setting.} We first evaluate the proposed approach in two problems, one is the classical Tiger problem~\cite{Cassandra1994}, the other is the Partially Observable Sysadmin~(POSyadmin)~\cite{Katt2017MCTS}. The same two problems are also the environments used to test the performance of the BA-POMCP approach~\cite{Katt2017MCTS}. For the POSyadmin problem, the agent acts as a system administrator to maintain a network of $n$ computers, which has $2n+1$ actions: 'ping' or ' reboot' any of the computers or 'do nothing', 3 observations: {NULL, failing, working}. The 'ping' action has a cost of 1, while rebooting a computer costs 20 and switches the computer to 'working', each 'failing' computer has a cost of 10 at each time step. The agent doesn't know the state of any computer, and at each time step, any of the computers can 'fail' with some probability $f$~\cite{Katt2017MCTS}. Then to further verify the effectiveness and scalability of the PSR model-based approach, we extend our approach to RockSample(5,5) and RockSample(5,7)~\cite{Ross2011BAPOMDP,Smith:2004:HSV}, both of which are too complex for the BA-POMDP based approaches. In the RockSample($n,k$) domain, a robot is on an $n \times n$ square board, with $k$ rocks on some of the cells. The positions of the robot and the rocks are known. Each rock has an unknown binary quality~(good or bad). The goal of the robot is to gather samples of the good rocks. The state of the robot is defined by the position of the robot on the board and the quality of all the rocks and there is an additional terminal state, reached when the robot moves into exit area, then with an $n \times n$ board and $k$ rocks, the number of states is $n^22^k+1$~\cite{Ross2011BAPOMDP}. Note that in the work of~\cite{Ross2011BAPOMDP}, the RockSample problem is also used but only with 37 states (RockSample(3,2)).

As in our proposed approach, some rewards are treated as observations. For the Tiger problem, besides the two observations of the original domain, the reward 10 for opening the correct door and the penalty $-100$ for choosing the door with the tiger behind it are also treated as observations. For the POSyadmin domain, we added the rewards that indicate the whole status of the network as observations, which provides the information about how many computers have been failed at current step, but we still don't know which computer has been failed, the same rewards were also used in the BA-POMCP based experiments. For the RockSample domains, we added the reward $10$ for sampling a good rock or moving into exit area and the penalty $-10$ for sampling a bad rock as observations.

For our approach, we first learned the PSR model of the underlying system offline, then the PSR model was combined with MCTS as shown in Algorithm~\ref{Algo:PSR-MCTS}$\sim$\ref{Algo:RollOut}. The PSR model can be learned straightforwardly. First, matrices $P_\mathcal{H}$,$P_{\mathcal{T,H}}$ and $P_{\mathcal{T},ao,\mathcal{H}}$ were estimated using the training data, then the model parameters can be computed using Equ.~\ref{equ:spectlearn}. The detail of the training data is as follows: For Tiger, $\mathcal{H}$ includes $200$ randomly generated trajectories, each containing $6$ action-observation pairs; $\mathcal{T}$ contains all the possible two-step action-observation pairs. For POSyadmin with $3/6$ computers, $\mathcal{H}$ includes $300/1000$ trajectories, and each containing $8/14$ action-observation pairs; $\mathcal{T}$ contains all the possible two-step/one-step action-observation pairs. For Rocksample(5,5)/Rocksample(5,7), $\mathcal{H}$ includes $600/7000$ trajectories, and each containing $20/23$ action-observation pairs; $\mathcal{T}$ contains all the possible two-step action-observation pairs. $p(t|h)$ of all matrices is estimated by executing the action-sequence of $t$ $50$ times. Nearly same amount of training data was used for the BA-POMCP approach.

\begin{figure*}[!ht]
\begin{minipage}{6.5in}
\begin{minipage}{3in}
\includegraphics[width=0.85\textwidth]{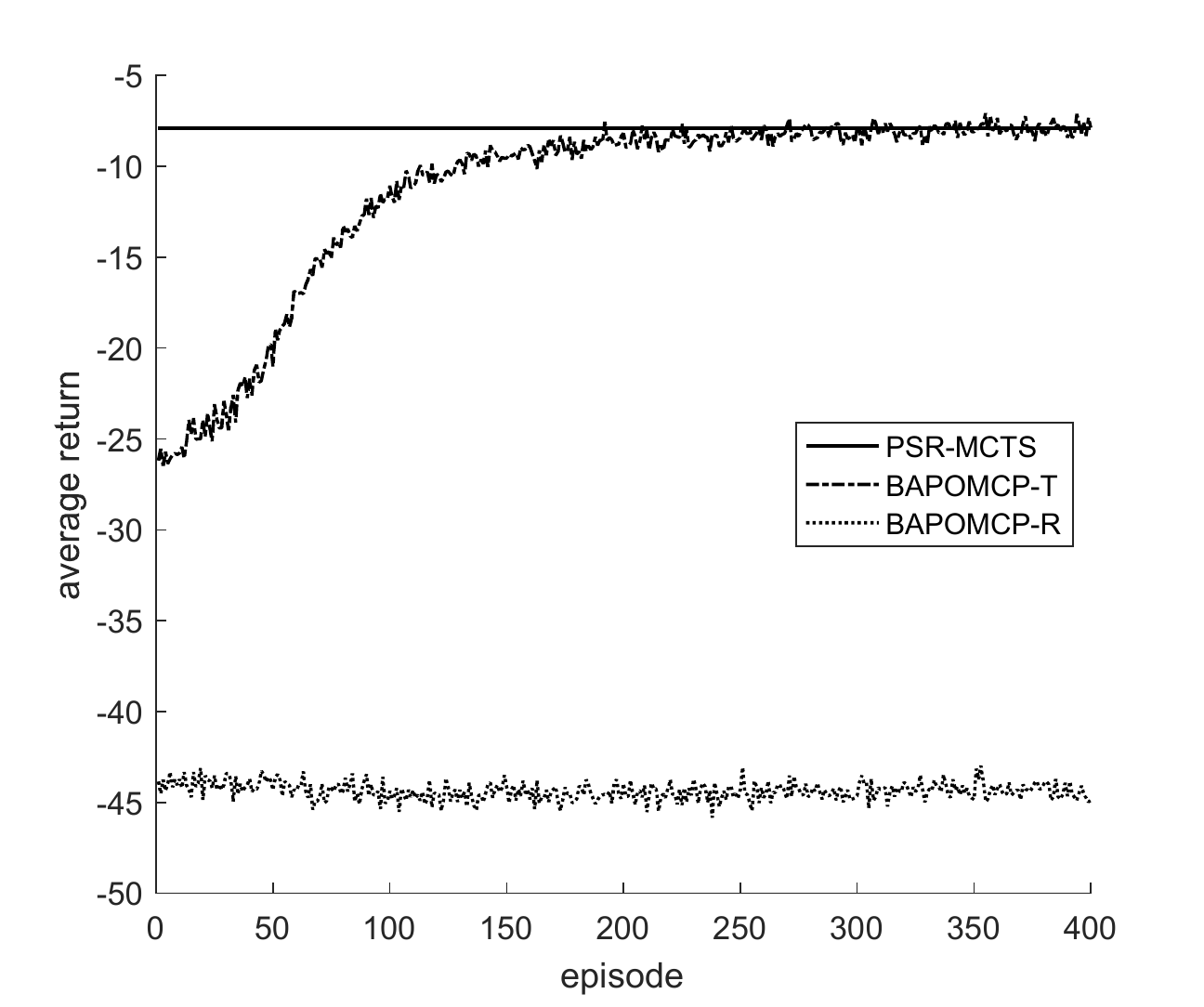}
\centerline{{\small ($a$)}}
\end{minipage}
\begin{minipage}{3in}
\includegraphics[width=0.85\textwidth]{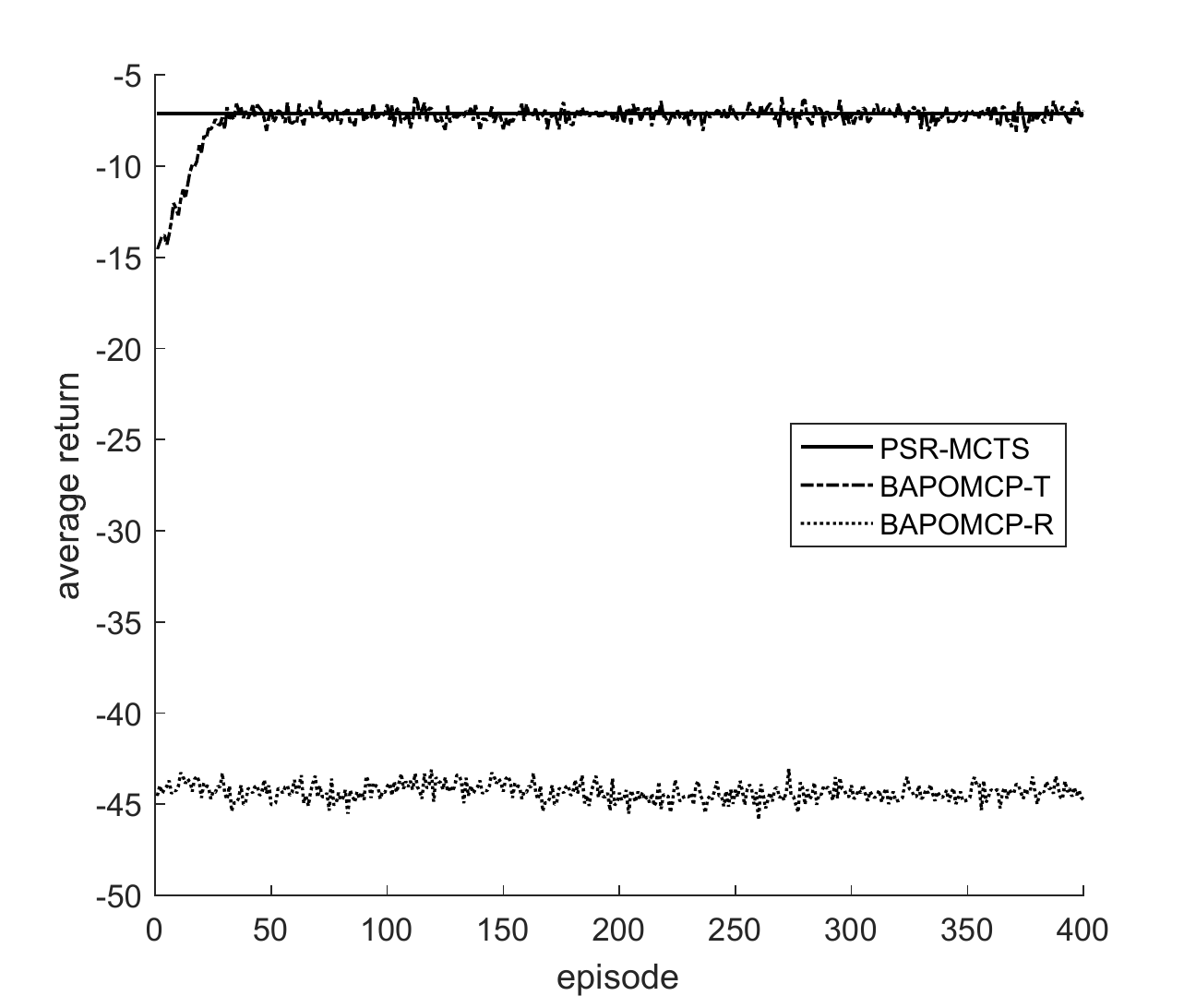}
\centerline{{\small ($b$)}}
\end{minipage}
\centering
\caption{{\small Average return for ($a$) 1000; ($b$) 10000 simulations on {\em Tiger}.}}
\label{fig:tiger}
\end{minipage}
\end{figure*}

\begin{figure*}[!ht]
\centering
\begin{minipage}{6.5in}
\begin{minipage}{3in}
\centering
\includegraphics[width=0.85\textwidth]{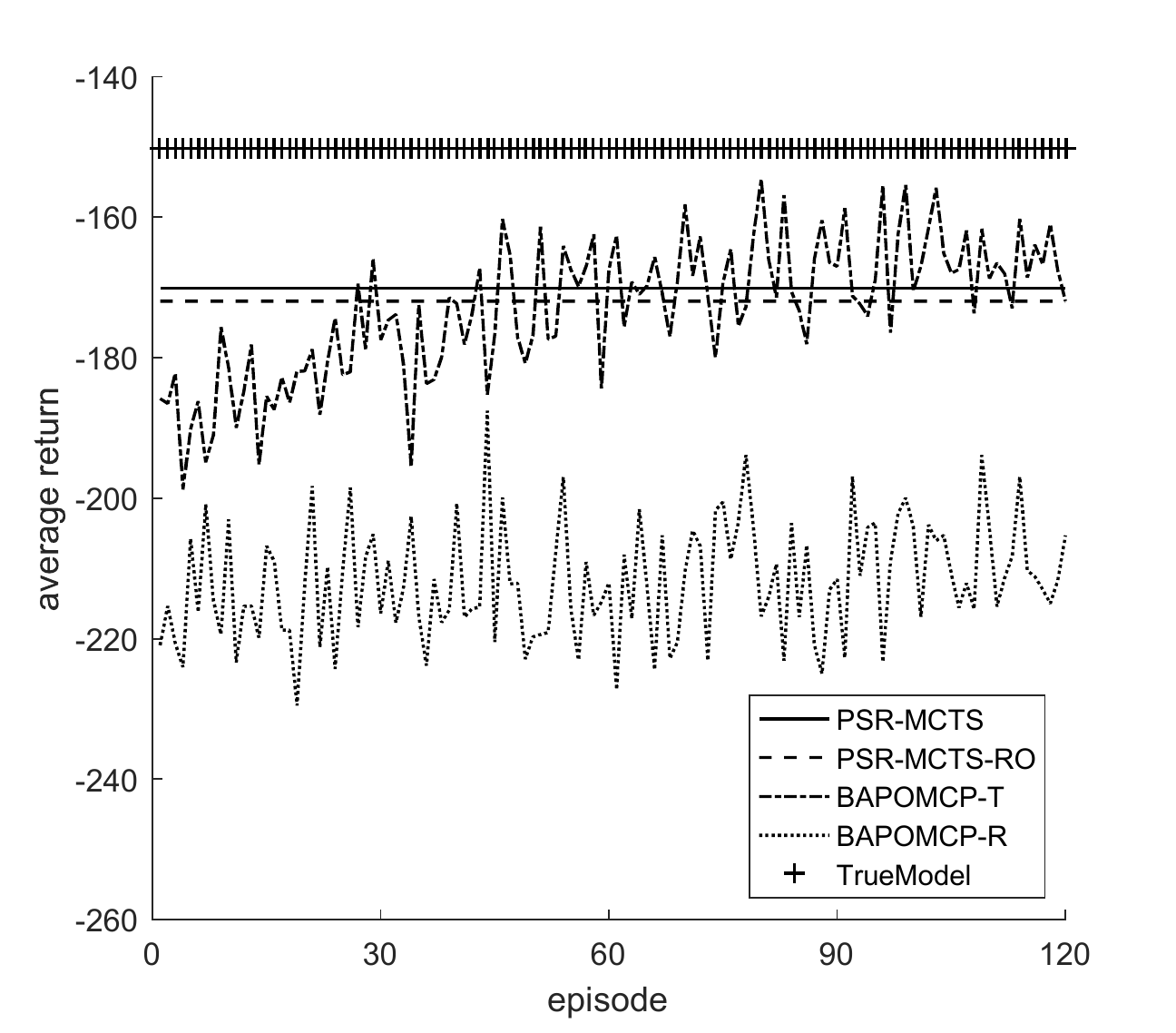}
\caption{{\small Average return on {\em POSyadmin}.}}
\label{fig:posyadmin3}
\end{minipage}
\begin{minipage}{3in}
\centering
\includegraphics[width=0.85\textwidth]{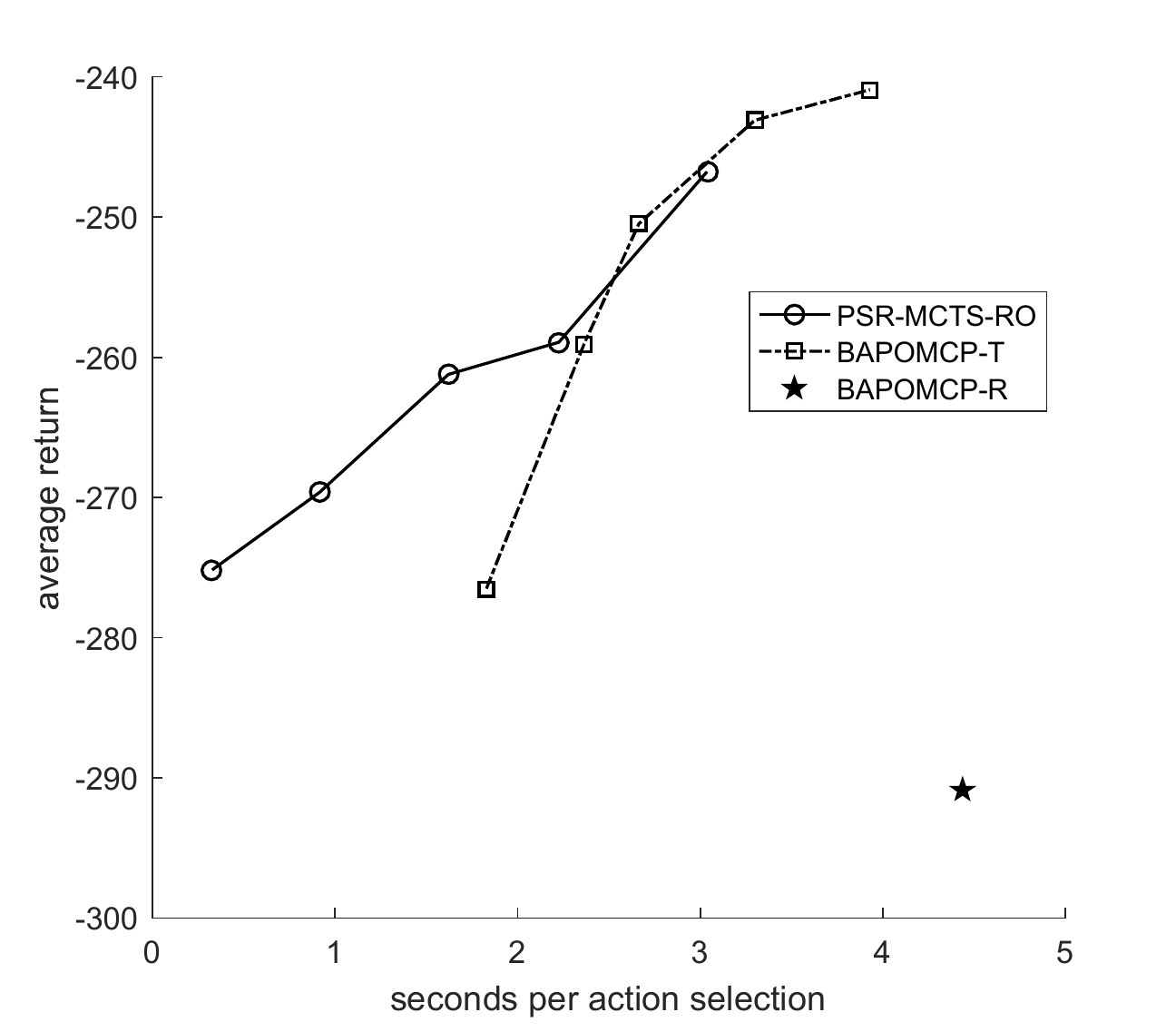}
\caption{{\small Average return per action selection time on {\em POSyadmin}.}}
\label{fig:posyadmin6}
\end{minipage}
\end{minipage}
\end{figure*}

\begin{figure*}[!ht]
\begin{minipage}{6.5in}
\begin{minipage}{3in}
\includegraphics[width=0.85\textwidth]{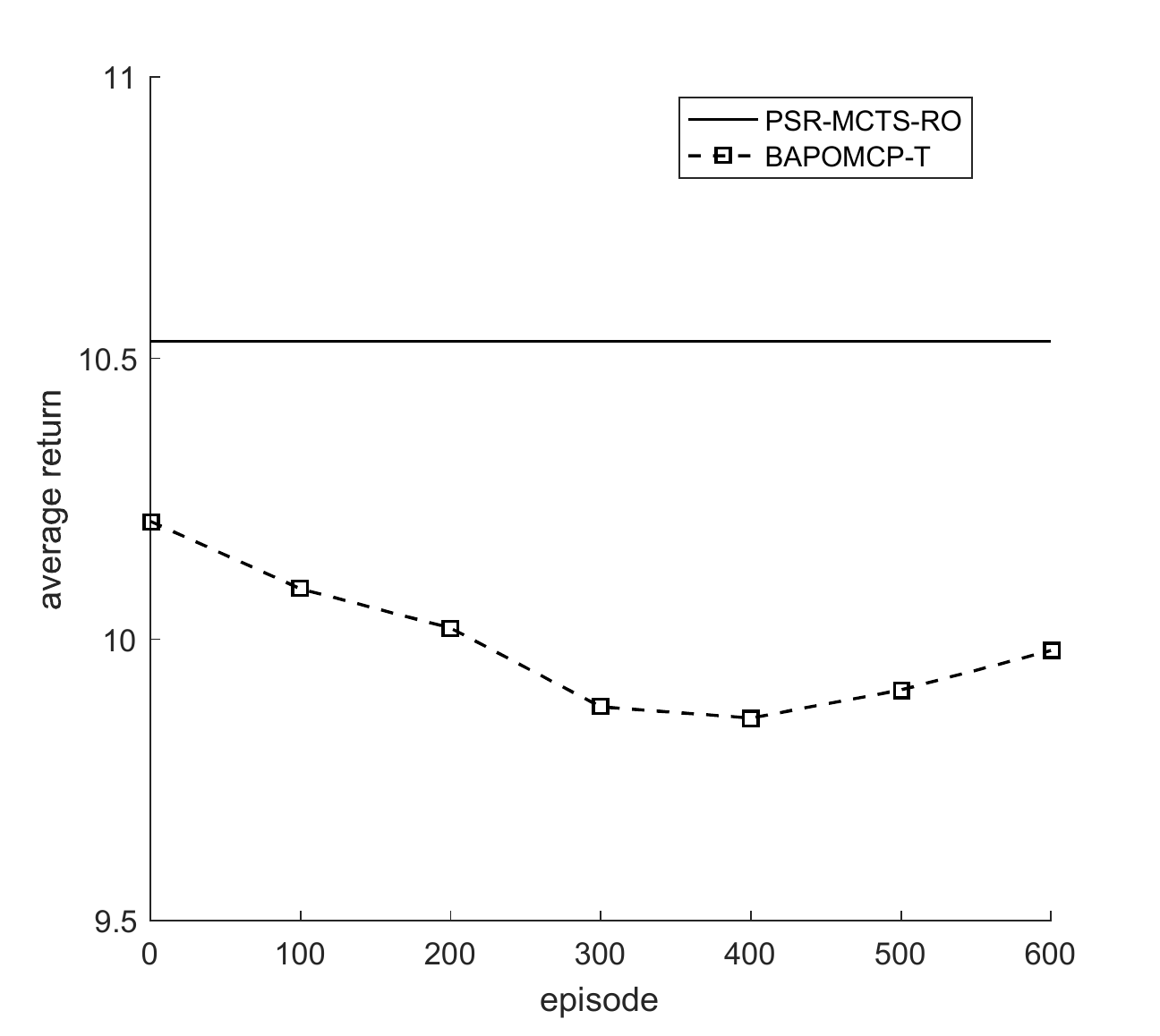}
\centerline{{\small ($a$)}}
\end{minipage}
\begin{minipage}{3in}
\includegraphics[width=0.85\textwidth]{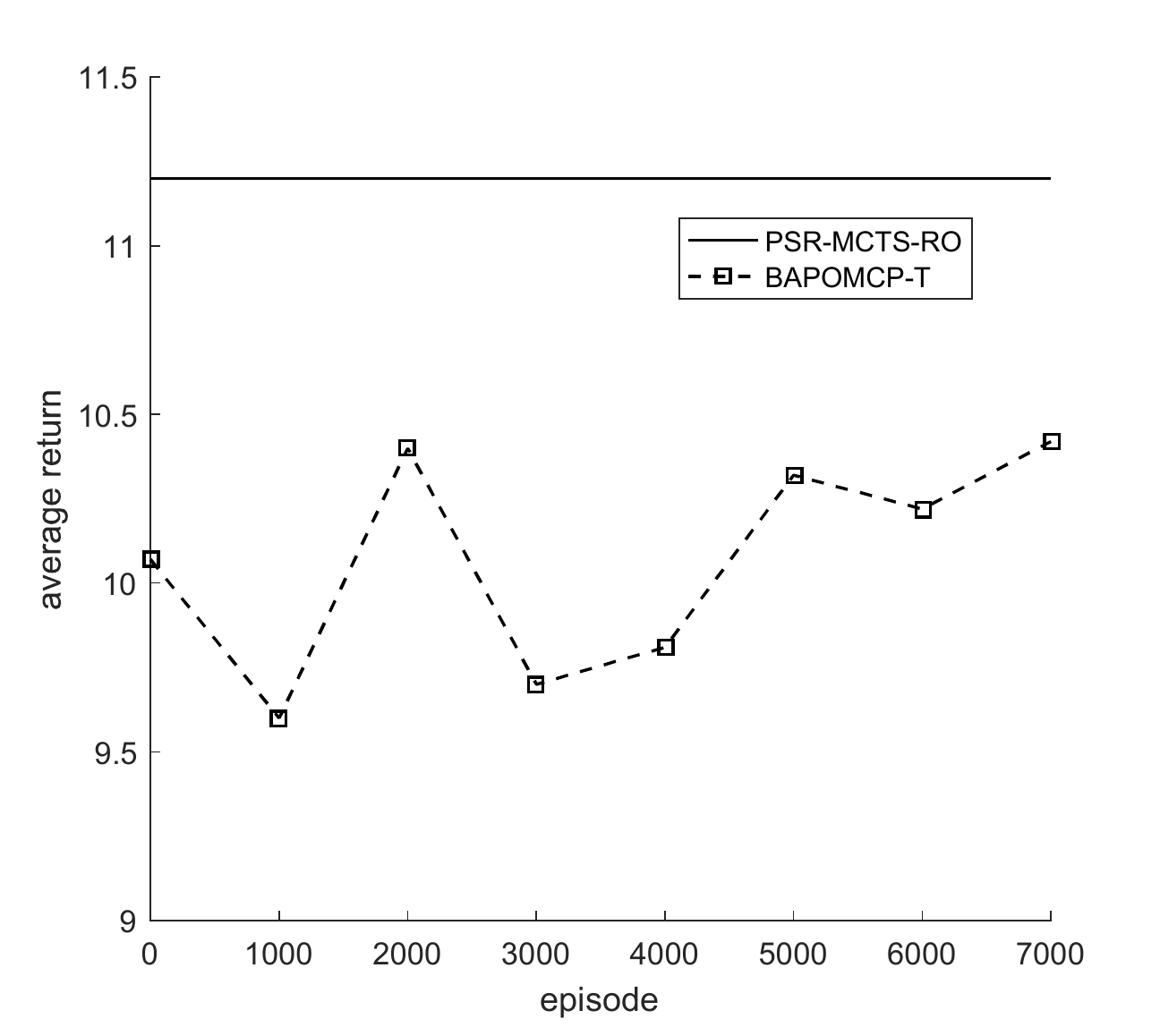}
\centerline{{\small ($b$)}}
\end{minipage}
\centering
\caption{{\small Average return on ($a$) {\em Rockample(5,5)}; ($b$) on {\em Rocksample(5,7)}.}}
\label{fig:rocksample}
\end{minipage}
\end{figure*}

\noindent{\bf Evaluated methods}. BA-POMCP is the most practical and state-of-the-art solution for BA-POMDPs~\cite{Katt2017MCTS,Katt2018}, however the performance of the BA-POMCP approach relies heavily on knowing the knowledge of the underlying system, for example, in the work of \cite{Katt2017MCTS}, for Tiger, the transition model is assumed to be correct and the initial observation function is assumed to be nearly correct; for POSyadmin, the observation function is assumed known a prior, and the initial transition function is assumed to be nearly correct; for RockSample, in the work of~\cite{Ross2011BAPOMDP}, the transition model is assumed to be correct and the initial observation function is assumed to be nearly correct. In practice, for many, if not most domains, such knowledge can be hardly known a prior.

To evaluate our method, for Tiger and POSyadmin, we firstly compared our approach to the BA-POMCP method under the same conditions~(BAPOMCP-R), that is, no prior knowledge is provided to the BA-POMCP approach, then our approach was compared to the BA-POMDP approach with the nearly correct initial models~(BAPOMCP-T) as mentioned previously. As for some large domains, the computation of every observation at each step is time-consuming, which may not meet the requirement of online planning, also the calculation of $P_{\mathcal{T},ao,\mathcal{H}}$ and $B_{ao}$ for all $a \in A$ and $o \in O$ may be computation-expensive. We further verified the performance of the PSR-MCTS approach in the case of reduced number of observations~(PSR-MCTS-RO), where only the $P_{\mathcal{T},ao,\mathcal{H}}$ and $B_{ao}$ for the $ao$ appeared in the training data were computed, and in the online planning process, when observation received from the world didn't belong to the set of observations appeared in the training data, an observation was randomly selected from this set and used for state update. For the RockSample problem, our approach~(PSR-MCTS-RO) was compared to the BA-POMCP approach with the nearly correct initial models~(BAPOMCP-T).

Same parameters were used as in the work of \cite{Katt2017MCTS}. For Tiger, POSyadmin and RockSample, the maximum number of decision steps for the agent is set to 20, 20, and 30 respectively. For the PSR-MCTS approach, for POSyadmin, the rank of PSR is set to 50; for RockSample(5,5)/RockSample(5,7) the rank of PSR is set to 70/75~(for Tiger, the rank is just the rank calculated from matrix $\hat P_\mathcal{T,H}$).

\noindent{\bf Performance evaluation.} Figure~\ref{fig:tiger} plots the average return over 10000 runs with 1000 and 10000 simulations on Tiger. Note that for both the Tiger and POSyadmin~(reported below) domains, for the BAPOMCP-R approach, nearly no improvement has been achieved with the increase of episodes and for the BAPOMCP-T approach, with the increase of the episodes, the performance becomes stable. For the PSR-MCTS approach, as mentioned previously, an offline model was first learned and no model learning is needed in the online planning process, the result reported is the average return and is shown as a line in the figures. As can be seen from the results, in all cases, when no prior knowledge is provided, the PSR-MCTS approach performs significantly better than the BA-POMCP approach. Even compared to the BA-POMCP approach with nearly correct initial models, the PSR-MCTS approach with no prior knowledge is still competitive, and different from our approach that no prior knowledge is required, prior knowledge plays a very important role for learning a good policy for the BA-POMCP approach.

Experimental results for five approaches on the (3-computer) POSysadmin problem are shown in Figure~\ref{fig:posyadmin3}, where 100 simulations per step were used. The PSR-MCTS related approaches still perform significantly better than the BAPOMCP-R approach and achieve nearly the same performance of the BAPOMCP-T approach. In Figure~\ref{fig:posyadmin3}, we also reported the result of the BA-POMCP approach given the completely accurate model of the underlying system, which also demonstrate the good performance of our approach. Even the PSR-MCTS-RO approach has a very good performance along with the significantly decreased average per action selection time, where the average per action selection time for BAPOMCP-R, BAPOMCP-T, PSR-MCTS, and PSR-MCTS-RO is about 0.03s, 0.02s, 0.45s and 0.19s respectively. However, as shown in the following experiment, the advantage for the BA-POMCP based approaches in terms of efficiency doesn't exist with the increase of the complexity of the underlying systems.

Figure~\ref{fig:posyadmin6} shows the average return over 100 runs with 100, 300, 500, 700 and 1000 simulations for the PSR-MCTS-RO and BAPOMCP-T approaches on POSyadmin with 6 computers. For the BAPOMCP-R approach, only the results of 100 simulations are given~(the single dot in the lower right corner) as the per action selection time for 100 simulations has reached more than 4 seconds while the return is much lower than the PSR-MCTS-RO approach. The results also show that compared to the BAPOMCP-T approach, the performance of the PSR-MCTS-RO approach is still good while with lower per action selection time at all cases. The explanation is that as for the BA-POMCP approaches, at each decision step, after taking an action and receiving an observation, the particle filter technique is used for approximating the next belief state, however, the obtaining of a next particle state that corresponding to the received observation is time-consuming for lager scale systems as only when the observation computed from a state and the corresponding model maps the real observation, this state can be used as the particle of next belief state. Note no results of the adaptations of the BA-POMCP approach is reported, as compared to the original BA-POMCP, the belief state update process with a random expected model for larger systems even needs more time and during the execution, no reasonable per action selection time required for online planning can be obtained. 

Figure~\ref{fig:rocksample} plots the average return over 1000 runs with 1000 simulations on RockSample(5,5) and RockSample(5,7). For such scale systems, the state size of the BA-POMDP model is intractable large with the increase of the time step. For the BA-POMDP based approaches, we may not even be able to store and initialize the state transaction matrices. For the comparison, as used in the work of \cite{Ross2011BAPOMDP}, only the dynamics related to the check action were modeled via the BA-POMDP approach, for the others, the black box simulation of the exact model was used to generate the simulation experiments and for state representation and updating~(BAPOMCP-T). Even under such conditions, as can be seen from the experimental results, the PSR-MCTS-RO method with no prior knowledge provided still achieved better performance compared with the BAPOMDP-T approach, and nearly no improvement has been achieved for the BAPOMCP-T approach with the arriving of new training data. The average per action selection time for PSR-MCTS-RO/BAPOMCP-T for RockSample(5,5) and RockSample(5,7) is about 0.93s/0.16s and 1.22s/1.16s respectively. The reason that less time has been used for the action selection of the BAPOMCP-based approaches is that the model representation of the BAPOMCP-based approaches in the experiment is not completely BA-POMDP based, and as mentioned earlier, a large part of the model in the related approaches is represented based on a black box. However, it can be seen that the action selection time of the BA-POMDP-based approaches is still highly affected by the scale of underlying system.

\section{Related Works}
Within the AI community, much attention has been devoted to solving the partially observable problem, i.e., the problem of planning under uncertainty. POMDPs provide a rich mathematical framework to solve it~\cite{Kaelbling98planningand,Ross2008OnlinePOMDP,Silver10monte-carloplanning,Ye2017}, however, most of the related algorithms assume the accurate POMDP models of the underlying systems are known a priori. And it is also known that learning offline POMDP models using some EM-like methods is very difficult and suffers from local minima, moreover the POMDP learning approaches usually presuppose knowledge of nature of the unobservable part of the world, which may be unrealistic in many real-world applications~\cite{Ye2017}. As an alternative, Predictive State Representations~(PSRs) provide a powerful framework for modelling partially observable and stochastic systems by only using observable quantities. Much effort has been devoted to learning offline PSR models. In the work of Boots {\em et al.}~\cite{Boots-2011b}, the offline PSR model is learned by using spectral approaches and under some assumptions, the spectral learning of PSRs has been proven to  be statistically consistent. Hamilton {\em et al.}~\cite{JMLR:hamilton14} presented the compressed PSR models, and the technique learns approximate PSR models by exploiting a particularly sparse structure presented in some domains, which allows for an increase in both the efficiency and predictive power.

When the model of the underlying system is available, model-based planning approaches offer a principled framework for solving the problem of choosing optimal actions in partially observable stochastic domains,e.g., in the work of \cite{Ye2017}, to overcome the challenges of ``curse of dimensionality'' and the ``curse of history'', the Determinized Sparse Partially Observable Tree~(DESPOT), a sparse approximation of the standard belief tree, for anytime online planning under uncertainty, was introduced, which focuses online planning on a set of randomly sampled scenarios and compactly captures the ``execution'' of all policies under these scenarios. However, as mentioned, most of the related methods assume an accurate model of the underlying system to be known a prior~(Note that rather than using MCTS, the DESPOT for the lookahead search can also be directly incorporated into our proposed framework for online planning). The BA-POMDP approach tackles this problem by using a Bayesian approach to model the distribution of all possible models and allows the models to be learned during execution~\cite{Ross2011BAPOMDP}, which has generated substantial interest in the literature~\cite{Ghavamzadeh2015BRLSurvey}. Unfortunately, BA-POMDPs are limited to some trivial problems as the size of the state space over all possible models is too large to be tractable for non-trivial problems~\cite{Katt2017MCTS,Katt2018}. In the PSR literature, in the work of~\cite{JMLR:hamilton14}, a compressed PSR~(CPSR) model is firstly learned and then the learned CPSR model is combined with Fitted-Q for planning. However, as mentioned, CPSR can be only applied to domains with a particularly sparse structure and some prior knowledge, e.g., domain knowledge, is still required.

With the benefits of online and sample-based planning for solving larger problems~\cite{Kearns2002,Ross2008OnlinePOMDP,Silver10monte-carloplanning,Ye2017}, some approaches have been proposed to solve the BA-POMDP model in an online manner. In the work of~\cite{Ross2011BAPOMDP}, an online POMDP solver is proposed by focusing on finding the optimal action to perform in the current belief of the agent. Katt et al.~\cite{Katt2017MCTS,Katt2018} extend the Monte-Carlo Tree Search method POMCP to BA-POMDPs, results in the state-of-the-art framework for learning and planning in BA-POMDPs. In the work of \cite{Katt2018}, a Factored Bayes-Adaptive POMDP model is introduced by exploiting the underlying structure of some specific domains. While these approaches show promising performance on some problems, like other Bayesian-based approaches in the literature, the performance is very dependent on the prior knowledge.

Given the accurate model of the environment to be known a prior, combining approximate offline and online solving approaches is an efficient way to tackle large POMDPs by using offline algorithms to compute lower and upper bounds on the optimal value function~\cite{Ross2007}. For the fully observable domains, in the work of \cite{Gelly2007}, offline and online value functions are combined in the UCT algorithm, where the offline value function is learned by using the $TD(\lambda)$ algorithm~\cite{Sutton1988} and used as prior knowledge in the UCT search tree, experimental results in a $9 \times 9$ Go program~(MoGo) demonstrates the effectiveness of such a combination.

\section{Conclusion and Future Work}
In this paper, we presented PSR-MCTS, a method for planning from scratch with model uncertainty, where an offline PSR model were firstly learned and then combined with online Monte-Carlo tree search. Through theoretical analysis and experiments, and by comparing to the state-of-the-art approach in the literature, we showed the effectiveness and efficiency of the proposed approach, moreover, our approach is more practical than other prior knowledge required approaches in many real-world applications. The modification of the original PSR-MCTS approach that using only a specific set of observations for model state updating are also proposed and tested. The effectiveness and scalability of our proposed approach are also tested on RockSample(5,5) and RockSample(5,7), which are infeasible for BA-POMDP based approaches. To our knowledge, our proposed approach is the first/only technique that have achieved an acceptable performance on the problem of planning with model uncertainty with no prior knowledge provided.

Future work includes developing more efficient techniques for state update and observation computation, and it is also interesting to apply online learning spectral methods for the PSR model learning, where the parameters of the PSR model can be updated during execution as the BA-POMDP approach has done.

\acks{This work was supported by the National Natural Science Foundation of China (No.61772438 and No.61375077).}

\vskip 0.2in
\bibliography{sample}

\begin{thebibliography}{}

\bibitem[\protect\BCAY{Balle, Carreras, Luque,\ \BBA\ Quattoni}{Balle
  et~al.}{2014}]{Balle:2014ML}
Balle, B., Carreras, X., Luque, F.~M., \BBA\ Quattoni, A. \BBOP2014\BBCP.
\newblock \BBOQ Spectral learning of weighted automata\BBCQ\
\newblock {\Bem Machine Learning}, {\Bem 96\/}(1-2), 33--63.

\bibitem[\protect\BCAY{Boots, Siddiqi,\ \BBA\ Gordon}{Boots
  et~al.}{2011}]{Boots-2011b}
Boots, B., Siddiqi, S.~M., \BBA\ Gordon, G.~J. \BBOP2011\BBCP.
\newblock \BBOQ Closing the learning-planning loop with predictive state
  representations\BBCQ\
\newblock {\Bem The International Journal of Robotics Research}, {\Bem
  30\/}(7), 954--966.

\bibitem[\protect\BCAY{Browne, Powley, Whitehouse, Lucas, Cowling, Rohlfshagen,
  Tavener, Perez, Samothrakis,\ \BBA\ Colton}{Browne
  et~al.}{2012}]{BrowneSurveyMCTS2012}
Browne, C.~B., Powley, E., Whitehouse, D., Lucas, S.~M., Cowling, P.~I.,
  Rohlfshagen, P., Tavener, S., Perez, D., Samothrakis, S., \BBA\ Colton, S.
  \BBOP2012\BBCP.
\newblock \BBOQ A survey of monte carlo tree search methods\BBCQ\
\newblock {\Bem IEEE Transactions on Computational Intelligence \& Ai in
  Games}, {\Bem 4\/}(1), 1--43.

\bibitem[\protect\BCAY{Cassandra, Kaelbling,\ \BBA\ Littman}{Cassandra
  et~al.}{1994}]{Cassandra1994}
Cassandra, A.~R., Kaelbling, L.~P., \BBA\ Littman, M.~L. \BBOP1994\BBCP.
\newblock \BBOQ Acting optimally in partially observable stochastic
  domains\BBCQ\
\newblock In {\Bem Twelfth National Conference on Artificial Intelligence},
  \BPGS\ 1023--1028.

\bibitem[\protect\BCAY{Duff}{Duff}{2002}]{Duff2002phdthesis}
Duff, M.~O. \BBOP2002\BBCP.
\newblock {\Bem Optimal Learning: Computational procedures for Bayes-adaptive
  Markov decision processes}.
\newblock Ph.D.\ thesis, University of Massachusetts at Amherst.

\bibitem[\protect\BCAY{Gelly, Kocsis, Schoenauer, Silver,\ \BBA\ Teytaud}{Gelly
  et~al.}{2012}]{GellyACM}
Gelly, S., Kocsis, L., Schoenauer, M., Silver, D., \BBA\ Teytaud, O.
  \BBOP2012\BBCP.
\newblock \BBOQ The grand challenge of computer go: Monte carlo tree search and
  extensions\BBCQ\
\newblock {\Bem Communications of the Acm}, {\Bem 55\/}(3), 106--113.

\bibitem[\protect\BCAY{Gelly\ \BBA\ Silver}{Gelly\ \BBA\
  Silver}{2007}]{Gelly2007}
Gelly, S.\BBACOMMA\  \BBA\ Silver, D. \BBOP2007\BBCP.
\newblock \BBOQ Combining online and offline knowledge in uct\BBCQ\
\newblock In {\Bem Proceedings of the 24th international conference on Machine
  learning}, \BPGS\ 273--280. ACM.

\bibitem[\protect\BCAY{Gelly\ \BBA\ Silver}{Gelly\ \BBA\
  Silver}{2011}]{Gelly2011AI}
Gelly, S.\BBACOMMA\  \BBA\ Silver, D. \BBOP2011\BBCP.
\newblock \BBOQ Monte-carlo tree search and rapid action value estimation in
  computer go\BBCQ\
\newblock {\Bem Artificial Intelligence}, {\Bem 175\/}(11), 1856--1875.

\bibitem[\protect\BCAY{Ghavamzadeh, Mannor, Pineau, Tamar, et~al.}{Ghavamzadeh
  et~al.}{2015}]{Ghavamzadeh2015BRLSurvey}
Ghavamzadeh, M., Mannor, S., Pineau, J., Tamar, A., et~al. \BBOP2015\BBCP.
\newblock \BBOQ Bayesian reinforcement learning: A survey\BBCQ\
\newblock {\Bem Foundations and Trends{\textregistered} in Machine Learning},
  {\Bem 8\/}(5-6), 359--483.

\bibitem[\protect\BCAY{Guez, Silver,\ \BBA\ Dayan}{Guez
  et~al.}{2013}]{Guez2013}
Guez, A., Silver, D., \BBA\ Dayan, P. \BBOP2013\BBCP.
\newblock \BBOQ Scalable and efficient bayes-adaptive reinforcement learning
  based on monte-carlo tree search\BBCQ\
\newblock {\Bem Journal of Artificial Intelligence Research}, {\Bem 48},
  841--883.

\bibitem[\protect\BCAY{Hamilton, Fard,\ \BBA\ Pineau}{Hamilton
  et~al.}{2014}]{JMLR:hamilton14}
Hamilton, W., Fard, M.~M., \BBA\ Pineau, J. \BBOP2014\BBCP.
\newblock \BBOQ Efficient learning and planning with compressed predictive
  states\BBCQ\
\newblock {\Bem The Journal of Machine Learning Research}, {\Bem 15\/}(1),
  3395--3439.

\bibitem[\protect\BCAY{Hsu, Kakade,\ \BBA\ Zhang}{Hsu
  et~al.}{2012}]{Hsu_aspectral}
Hsu, D., Kakade, S.~M., \BBA\ Zhang, T. \BBOP2012\BBCP.
\newblock \BBOQ A spectral algorithm for learning hidden markov models\BBCQ\
\newblock {\Bem Journal of Computer \& System Sciences}, {\Bem 78\/}(5),
  1460--1480.

\bibitem[\protect\BCAY{Huang, An, Zhou, Hong,\ \BBA\ Liu}{Huang
  et~al.}{2018}]{Huang2018}
Huang, C., An, Y., Zhou, S., Hong, Z., \BBA\ Liu, Y. \BBOP2018\BBCP.
\newblock \BBOQ Basis selection in spectral learning of predictive state
  representations\BBCQ\
\newblock {\Bem Neurocomputing}, {\Bem 310}, 183 -- 189.

\bibitem[\protect\BCAY{Kaelbling, Littman,\ \BBA\ Cassandra}{Kaelbling
  et~al.}{1998}]{Kaelbling98planningand}
Kaelbling, L.~P., Littman, M.~L., \BBA\ Cassandra, A.~R. \BBOP1998\BBCP.
\newblock \BBOQ Planning and acting in partially observable stochastic
  domains\BBCQ\
\newblock {\Bem Artificial intelligence}, {\Bem 101\/}(1-2), 99--134.

\bibitem[\protect\BCAY{Katt, Oliehoek,\ \BBA\ Amato}{Katt
  et~al.}{2017}]{Katt2017MCTS}
Katt, S., Oliehoek, F.~A., \BBA\ Amato, C. \BBOP2017\BBCP.
\newblock \BBOQ Learning in pomdps with monte carlo tree search\BBCQ\
\newblock In Precup, D.\BBACOMMA\  \BBA\ Teh, Y.~W.\BEDS, {\Bem Proceedings of
  the 34th International Conference on Machine Learning}, \lowercase{\BVOL}~70
  of {\Bem Proceedings of Machine Learning Research}, \BPGS\ 1819--1827,
  International Convention Centre, Sydney, Australia. PMLR.

\bibitem[\protect\BCAY{Katt, Oliehoek,\ \BBA\ Amato}{Katt
  et~al.}{2018}]{Katt2018}
Katt, S., Oliehoek, F.~A., \BBA\ Amato, C. \BBOP2018\BBCP.
\newblock \BBOQ Bayesian reinforcement learning in factored pomdps\BBCQ\
\newblock {\Bem CoRR}, {\Bem abs/1811.05612}.

\bibitem[\protect\BCAY{Kearns, Mansour,\ \BBA\ Ng}{Kearns
  et~al.}{2002}]{Kearns2002}
Kearns, M., Mansour, Y., \BBA\ Ng, A.~Y. \BBOP2002\BBCP.
\newblock \BBOQ A sparse sampling algorithm for near-optimal planning in large
  markov decision processes\BBCQ\
\newblock {\Bem Machine learning}, {\Bem 49\/}(2-3), 193--208.

\bibitem[\protect\BCAY{Kocsis\ \BBA\ Szepesv{\'a}ri}{Kocsis\ \BBA\
  Szepesv{\'a}ri}{2006}]{Kocsis:2006}
Kocsis, L.\BBACOMMA\  \BBA\ Szepesv{\'a}ri, C. \BBOP2006\BBCP.
\newblock \BBOQ Bandit based monte-carlo planning\BBCQ\
\newblock In {\Bem European conference on machine learning}, \BPGS\ 282--293.
  Springer.

\bibitem[\protect\BCAY{Littman, Sutton,\ \BBA\ Singh}{Littman
  et~al.}{2001}]{Littman01predictiverepresentations}
Littman, M.~L., Sutton, R.~S., \BBA\ Singh, S. \BBOP2001\BBCP.
\newblock \BBOQ Predictive representations of state\BBCQ\
\newblock In {\Bem Proceedings of the 14th International Conference on Neural
  Information Processing Systems: Natural and Synthetic}, \BPGS\ 1555--1561.
  MIT Press.

\bibitem[\protect\BCAY{Liu, Tang,\ \BBA\ Zeng}{Liu et~al.}{2015}]{LiuAAMAS15}
Liu, Y., Tang, Y., \BBA\ Zeng, Y. \BBOP2015\BBCP.
\newblock \BBOQ Predictive state representations with state space
  partitioning\BBCQ\
\newblock In {\Bem Proceedings of the 2015 International Conference on
  Autonomous Agents and Multiagent Systems~(AAMAS)}, \BPGS\ 1259--1266.

\bibitem[\protect\BCAY{Liu, Yang,\ \BBA\ Ji}{Liu
  et~al.}{2014}]{Liu14inaccuratePSR}
Liu, Y., Yang, Z., \BBA\ Ji, G. \BBOP2014\BBCP.
\newblock \BBOQ Solving partially observable problems with inaccurate psr
  models\BBCQ\
\newblock {\Bem Information Sciences}, {\Bem 283}, 142--152.

\bibitem[\protect\BCAY{Liu, Zhu, Zeng,\ \BBA\ Dai}{Liu
  et~al.}{2016}]{Liu16IJCAI}
Liu, Y., Zhu, H., Zeng, Y., \BBA\ Dai, Z. \BBOP2016\BBCP.
\newblock \BBOQ Learning predictive state representations via monte-carlo tree
  search\BBCQ\
\newblock In {\Bem Proceedings of the 25th International Joint Conference on
  Artificial Intelligence~(IJCAI)}.

\bibitem[\protect\BCAY{Pineau, Gordon,\ \BBA\ Thrun}{Pineau
  et~al.}{2006}]{Pineau2006}
Pineau, J., Gordon, G., \BBA\ Thrun, S. \BBOP2006\BBCP.
\newblock \BBOQ Anytime point-based approximations for large pomdps\BBCQ\
\newblock {\Bem Journal of Artificial Intelligence Research}, {\Bem 27}, 2006.

\bibitem[\protect\BCAY{Ross, Chaib-draa,\ \BBA\ Pineau}{Ross
  et~al.}{2008a}]{Ross2008BAPOMDP}
Ross, S., Chaib-draa, B., \BBA\ Pineau, J. \BBOP2008a\BBCP.
\newblock \BBOQ Bayes-adaptive pomdps\BBCQ\
\newblock In {\Bem Advances in neural information processing systems}, \BPGS\
  1225--1232.

\bibitem[\protect\BCAY{Ross, Pineau,\ \BBA\ Chaib-draa}{Ross
  et~al.}{2008b}]{Ross2007}
Ross, S., Pineau, J., \BBA\ Chaib-draa, B. \BBOP2008b\BBCP.
\newblock \BBOQ Theoretical analysis of heuristic search methods for online
  pomdps\BBCQ\
\newblock In {\Bem Advances in neural information processing systems}, \BPGS\
  1233--1240.

\bibitem[\protect\BCAY{Ross, Pineau, Chaib-draa,\ \BBA\ Kreitmann}{Ross
  et~al.}{2011}]{Ross2011BAPOMDP}
Ross, S., Pineau, J., Chaib-draa, B., \BBA\ Kreitmann, P. \BBOP2011\BBCP.
\newblock \BBOQ A bayesian approach for learning and planning in partially
  observable markov decision processes\BBCQ\
\newblock {\Bem Journal of Machine Learning Research}, {\Bem 12\/}(May),
  1729--1770.

\bibitem[\protect\BCAY{Ross, Pineau, Paquet,\ \BBA\ Chaib-Draa}{Ross
  et~al.}{2008}]{Ross2008OnlinePOMDP}
Ross, S., Pineau, J., Paquet, S., \BBA\ Chaib-Draa, B. \BBOP2008\BBCP.
\newblock \BBOQ Online planning algorithms for pomdps\BBCQ\
\newblock {\Bem Journal of Artificial Intelligence Research}, {\Bem 32},
  663--704.

\bibitem[\protect\BCAY{Silver\ \BBA\ Veness}{Silver\ \BBA\
  Veness}{2010}]{Silver10monte-carloplanning}
Silver, D.\BBACOMMA\  \BBA\ Veness, J. \BBOP2010\BBCP.
\newblock \BBOQ Monte-carlo planning in large pomdps\BBCQ\
\newblock In {\Bem Advances in neural information processing systems}, \BPGS\
  2164--2172.

\bibitem[\protect\BCAY{Singh, James,\ \BBA\ Rudary}{Singh
  et~al.}{2004}]{Singh04predictivestate}
Singh, S., James, M.~R., \BBA\ Rudary, M.~R. \BBOP2004\BBCP.
\newblock \BBOQ Predictive state representations: A new theory for modeling
  dynamical systems\BBCQ\
\newblock In {\Bem Proceedings of the 20th conference on Uncertainty in
  artificial intelligence}, \BPGS\ 512--519. AUAI Press.

\bibitem[\protect\BCAY{Smith\ \BBA\ Simmons}{Smith\ \BBA\
  Simmons}{2004}]{Smith:2004:HSV}
Smith, T.\BBACOMMA\  \BBA\ Simmons, R. \BBOP2004\BBCP.
\newblock \BBOQ Heuristic search value iteration for pomdps\BBCQ\
\newblock In {\Bem Proceedings of the 20th Conference on Uncertainty in
  Artificial Intelligence}, UAI '04, \BPGS\ 520--527, Arlington, Virginia,
  United States. AUAI Press.

\bibitem[\protect\BCAY{Spaan\ \BBA\ Vlassis}{Spaan\ \BBA\
  Vlassis}{2005}]{Spaan2005}
Spaan, M. T.~J.\BBACOMMA\  \BBA\ Vlassis, N. \BBOP2005\BBCP.
\newblock \BBOQ Perseus: Randomized point-based value iteration for
  pomdps\BBCQ\
\newblock {\Bem Journal of Artificial Intelligence Research}, {\Bem 24},
  195--220.

\bibitem[\protect\BCAY{Sutton}{Sutton}{1988}]{Sutton1988}
Sutton, R.~S. \BBOP1988\BBCP.
\newblock \BBOQ Learning to predict by the methods of temporal
  differences\BBCQ\
\newblock {\Bem Machine Learning}, {\Bem 3\/}(1), 9--44.

\bibitem[\protect\BCAY{Talvitie\ \BBA\ Singh}{Talvitie\ \BBA\
  Singh}{2011}]{TalvitieS11}
Talvitie, E.\BBACOMMA\  \BBA\ Singh, S.~P. \BBOP2011\BBCP.
\newblock \BBOQ Learning to make predictions in partially observable
  environments without a generative model\BBCQ\
\newblock {\Bem J. Artif. Intell. Res. (JAIR)}, {\Bem 42}, 353--392.

\bibitem[\protect\BCAY{Ye, Somani, Hsu,\ \BBA\ Lee}{Ye et~al.}{2017}]{Ye2017}
Ye, N., Somani, A., Hsu, D., \BBA\ Lee, W.~S. \BBOP2017\BBCP.
\newblock \BBOQ Despot: Online pomdp planning with regularization\BBCQ\
\newblock {\Bem J. Artif. Int. Res.}, {\Bem 58\/}(1), 231--266.

\end{thebibliography}
\bibliographystyle{theapa}

\end{document}